\documentclass{siamart220329}
\usepackage{amsmath,amssymb,booktabs,graphicx,enumitem,float}
\usepackage{algpseudocode}
\newcommand{\Om}{\mathcal{O}(m)}
\newcommand{\R}{\mathbb{R}}
\hypersetup{colorlinks=true,allcolors=blue!55!black}
\makeatletter
\AtBeginDocument{%
  \def\refstepcounter@optarg[#1]#2{%
    \cref@old@refstepcounter{#2}%
    \cref@constructprefix{#2}{\cref@result}%
    \@ifundefined{cref@#1@alias}%
      {\def\@tempa{#1}}%
      {\def\@tempa{\csname cref@#1@alias\endcsname}}%
    \protected@edef\cref@currentlabel{%
      [\@tempa][\arabic{#2}][\cref@result]%
      \csname p@#2\endcsname\csname the#2\endcsname}%
  }%
}
\makeatother

\title{A Near-Linear-Time Solver for Graph $p$-Laplacian Semi-Supervised
Learning via Continuation in $p$\thanks{Submitted to the editors \today.}}
\author{Oren E. Livne\thanks{Pine Birch Analytics, Churchville, PA
(\texttt{oren.livne@gmail.com}, \texttt{pinebirchanalytics.com}; ORCID:
\href{https://orcid.org/0000-0001-6700-483X}{0000-0001-6700-483X}).}}
\headers{Near-linear-time graph $p$-Laplacian semi-supervised learning}{O. E. Livne}

\begin{document}
\maketitle

\begin{abstract}
Graph-based semi-supervised learning (SSL) propagates a few labels over a similarity graph by
minimizing a Dirichlet-type energy. The standard quadratic ($p=2$) energy reduces to a single
graph-Laplacian solve, but it \emph{degenerates} exactly where SSL is most useful---when labels are
scarce: gathering more unlabeled data drives the $p=2$ estimate to a near-constant function whenever
$d\ge2$ (Nadler--Srebro--Zhou). Well-posedness requires the \emph{nonlinear} $p$-Laplacian energy
$\sum_e w_e|x_i-x_j|^p$ with $p>d$. Existing solvers reduce this to a sequence of weighted
Laplacian solves, but their reference
implementations use a direct sparse factorization or ichol-preconditioned CG instead.
Plugging a near-linear Laplacian solver isn't straightforward: at large $p$ the Laplacian
conductance weights can degenerate near flat-gradient edges, making
the system nearly singular and causing an inexact inner solve to stagnate without a damped outer
iteration to control conditioning. We close this gap. Recasting $p$-Laplacian SSL as a source-form
\emph{nonlinear Laplacian flow} $B\rho_p(B^\top x)=b$ and solving by damped chord-Newton
continuation in $p$, every linearized system stays well-conditioned and can be delegated to a
near-linear Laplacian engine. On
size-scaled graph families the wall-clock is empirically linear ($m^{0.96}$--$m^{1.02}$ per family with the LAMG$+$ engine;
approximate Cholesky is near-linear but faster, thus is the default), and a single fit pooled across a corpus of 228
real SuiteSparse graphs gives $m^{1.19}$ (vs.\ $m^{1.45}$ for direct factorization), with a nearly flat
Newton-step count; the solver handles a $6.8\times10^7$-edge social network in minutes---orders of magnitude
beyond the direct-factorization ceiling. The bottleneck is
memory: a sparse Cholesky factor on irregular SSL graphs fills $10$--$280\times$ the graph's
nonzeros vs.\ our $\mathcal{O}(m)$ hierarchy, making the direct approach infeasible at scale
regardless of speed. Furthermore, the speed crossover occurs at
$m\sim10^4$ edges. Against the released FCL variational solver on kNN graphs we are $1.5$--$14\times$ faster at matched accuracy, with a margin that grows with graph size and is set by the inner-solver scaling rather than by implementation.
On the canonical MNIST $10$-NN benchmark, $p=3$ scores $64\%$ at one label
per class versus $36\%$ for $p=2$---a direct realization of the degeneracy fix.
Continuation in $p$ is not itself new; our contribution is pairing it with a fast inner engine that makes
graph $p$-Laplacian learning usable at web scale. Code: \url{https://github.com/orenlivne/np}.
\end{abstract}



\section{Introduction}\label{sec:intro}
Semi-supervised learning (SSL) on graphs is a workhorse of modern data analysis: given a similarity
graph over $n$ data points and labels on a small seed set, propagate the labels to score every node.
The classical formulation (Zhu--Ghahramani--Lafferty; Zhou et al.) minimizes the quadratic Dirichlet
energy
\begin{equation}\label{eq:quad}
E_2(x)=\tfrac12\sum_{(i,j)\in E} w_{ij}\,(x_i-x_j)^2 \;+\;\tfrac{\mu}{2}\sum_{i\in S}(x_i-y_i)^2 ,
\end{equation}
whose minimizer solves a single symmetric diagonally dominant (SDDM) linear system---a graph Laplacian
plus a diagonal label term. Fast near-linear Laplacian solvers (algebraic multigrid, combinatorial
multigrid, approximate Cholesky (AC)) make \eqref{eq:quad} cheap even at web scale, so for the quadratic
model the solver question is essentially closed.

\subsection{The low-label degeneracy}
The quadratic model has a well-documented failure mode exactly where SSL is most valuable---when labels
are scarce. Nadler, Srebro and Zhou~\cite{nsz2009} proved that with a \emph{fixed} number of labeled
points and the number of unlabeled points $n\to\infty$, the $p=2$ estimate is \emph{ill-posed in
dimension $d\ge2$}: its minimizer converges to a constant away from the labels, decorated by vanishing
spikes at the seeds, and carries no information. The mechanism is capacity: a point constraint has zero
capacity for the Dirichlet ($H^1$) energy in $d\ge2$, so the labels become invisible to the continuum
energy. The remedy is to raise the exponent, which penalizes large gradients more. Replacing \eqref{eq:quad} by the $p$-Dirichlet
energy
\begin{equation}\label{eq:penergy}
E_p(x)=\frac1p\sum_{(i,j)\in E} w_{ij}\,|x_i-x_j|^p \;+\;\frac{\mu}{2}\sum_{i\in S}(x_i-y_i)^2 ,
\end{equation}
whose Euler--Lagrange operator is the graph $p$-Laplacian, restores positive capacity of point
constraints exactly when $p>d$: Slep\v{c}ev--Thorpe~\cite{slepcev2019} show the discrete functional
$\Gamma$-converges to a nondegenerate continuum limit that retains the labels iff $p>d$, and
Calder~\cite{calder2019} proves the game-theoretic $p$-Laplacian is well-posed with arbitrarily few
labels, the same $p>d$ threshold appearing through the barrier $|x-y|^{(p-d)/(p-1)}$ (continuous iff
$p>d$). El~Alaoui et al.~\cite{elalaoui2016} identify a companion smoothness transition at $p=d+1$.
The empirical effect is dramatic: at one label per class on MNIST, label propagation ($p=2$) is near
chance while nonlinear-$p$ / Poisson variants exceed $90\%$~\cite{calder2020}.

\subsection{The solver gap}
The nonlinear energy \eqref{eq:penergy} is convex but not quadratic, and this is where scalability
breaks down. The methods regarded as fastest and most robust---Newton's method with homotopy in $p$
(Flores--Calder--Lerman~\cite{fcl2022}) and the recent dual-IRLS scheme (Storn~\cite{storn2026}), which
restores linear convergence for large $p$ where primal IRLS diverges ($p\ge3$)---share one structure:
each outer step is a weighted graph-Laplacian (i.e.\ $p=2$-type) solve. That inner solve is the
entire cost, yet their reference implementations perform it by direct sparse factorization, which is superlinear in $n$ and caps these variational-energy solvers near
$10^4$--$10^5$ nodes (\S\ref{sec:scaling}). One might expect the fix to be a one-line substitution---%
drop a near-linear Laplacian solver in for the factorization---but it is not, for two reasons. Primal
IRLS diverges for $p\ge3$~\cite{storn2026}, so in the very regime SSL needs ($p>d\ge2$) there is
no convergent outer loop to host a fast inner solve at all. For the methods that do converge
(Newton-homotopy, dual IRLS), the substitution is not a drop-in: at large $p$ the reweighted
conductances $w_e|x_i-x_j|^{p-2}$ vanish at flat-gradient edges, so each linearized Laplacian
degenerates toward singular, and an inexact near-linear inner solve stagnates unless a damped
outer iteration holds its conditioning. Doing this substitution correctly, which is our
contribution, requires both a continuation and a fast linear solver. Calder's \texttt{GraphLearning}, the one widely-used package that does scale in
memory, solves a different operator, the game-theoretic normalized $p$-Laplacian, by a matrix-free fixed-point iteration. It avoids the factorization but its $\infty$-Laplacian sweep converges slowly on high-diameter graphs; it is $5.5$--$23\times$ slower than
our solver on geometric graphs (\S\ref{sec:scaling}). On the
theory side, the celebrated almost-linear-time $p$-norm-flow and max-flow results
\cite{kyng2019,chen2022} are, like Spielman--Teng's original near-linear Laplacian solver before
practical algebraic multigrid, purely theoretical with no usable implementation. There is no
``AC equivalent'' for $p$-Laplacian SSL yet.

\subsection{Contribution}
We supply the missing scalable engine, reusing a machinery built for a different nonlinear-flow problem.
Our recent \emph{Nonlinear Laplacian Flow} (NLF) solver~\cite{nlf}\footnote{The solver machinery is
restated in \S\ref{sec:cont} and its global convergence is proved self-contained in
Appendix~\ref{app:conv} from Proposition~\ref{prop:reduce} alone.} solves $B\rho(B^\top\phi)=\alpha d$
for a monotone edge law $\rho$ by an inexact chord-Newton iteration whose frozen linearization is a
weighted graph Laplacian, inverted by a near-linear Laplacian engine (AC by default; LAMG$+$ interchangeable). We show that $p$-Laplacian SSL is
an instance of exactly this equation and inherits its near-linear inner solve. Concretely:
\begin{enumerate}[leftmargin=*]
\item \textbf{Reduction (\S\ref{sec:method}).} We recast \eqref{eq:penergy} as a source-form nonlinear
Laplacian flow $B\rho_p(B^\top x)=b$ by folding the label term into a single grounded node carrying
linear ``mass'' edges. The graph $p$-Laplacian edge law $\rho_p(s)=|s|^{p-1}\mathrm{sign}(s)$ is
monotone, so NLF applies verbatim; no bespoke solver is written.
\item \textbf{Continuation in $p$ (\S\ref{sec:cont}).} We march $p$ from the $p=2$ linear solve to the target, warm-starting and sharing one lazily refreshed hierarchy. Continuation in
$p$ is standard~\cite{fcl2022,storn2026}; the point is that each stage costs a handful of near-linear
solves.
\item \textbf{Near-linear scaling (\S\ref{sec:scaling}).} At a fixed graph type scaled in size, the
wall-clock is empirically $\Om$ with the LAMG+ engine ($m^{0.96}$--$m^{1.02}$) and near-linear with the
faster AC default ($m^{1.10}$--$m^{1.18}$, its $\mathcal{O}(m\log^{3}n)$ factor); a
single fit pooled over the heterogeneous $228$-graph corpus gives $m^{1.19}$, a mixed-class descriptive
aggregate. Holding the outer method fixed and swapping only the inner
solve, the near-linear engine
(AC by default; LAMG$+$ interchangeable) pulls away from the direct sparse
factorization used by the incumbents' reference implementations, and finishes where the
factorization is fill-blown.
\item \textbf{The degeneracy fix, at scale (\S\ref{sec:accuracy}).} On the canonical MNIST $10$-NN graph
our solver reproduces the low-label degeneracy fix end to end.
\end{enumerate}

\emph{What is new here versus~\cite{nlf}.} The NLF solver itself---the chord-Newton iteration, the lazily
refreshed hierarchy, the multicommodity reuse, and the guarded Anderson step---is prior work living in
the NLF layer; we import it wholesale and write no bespoke solver. New to this paper are (i)~the
grounded-node reduction (Proposition~\ref{prop:reduce}) that makes $p$-Laplacian SSL a verbatim instance
of the NLF master equation; (ii)~the continuation-in-$p$ setup for SSL together with the diagnosis of
the large-$p$ conductance-floor stall mode, which the NLF solver's guarded Anderson step (invoked here)
resolves to $100\%$ corpus convergence; and (iii)~the empirical demonstration that this combination
solves, at web scale and to our knowledge for the first time, a problem the direct-factorization
incumbents could not.

\subsection{Why this matters for machine learning}
Graph label propagation is not a historical curiosity. It underpins large-scale industrial pipelines
for recommendation, spam and abuse detection, and content categorization on graphs of $10^8$--$10^9$
edges, and it has re-emerged as a competitive, inexpensive alternative to graph neural networks: the
``Correct and Smooth'' method~\cite{huang2021} shows a plain multilayer perceptron plus label
propagation matching or beating many GNNs on the Open Graph Benchmark at a fraction of the training
cost, and label propagation continues to be revisited as a competitive alternative~\cite{cheng2024}. The nonlinear $p$-Laplacian variant is exactly the tool for the \emph{extreme low-label} regime
(few-shot, weak supervision, active learning), where the quadratic model degenerates
(\S\ref{sec:accuracy}) and feature-hungry GNNs overfit. What has been missing is a solver that runs it
at the scale these applications demand---the direct-factorization incumbents cap at four orders of magnitude below the graphs that need it. Removing that ceiling, at equal
accuracy, is the practical contribution; we do not claim the $p$-Laplacian dominates GNNs when rich
features and ample labels are available, only that it is the right tool at low labels and that we make
it scalable.

\section{The $p$-Laplacian SSL flow}\label{sec:method}
\subsection{Stationarity}
Let $B\in\R^{n\times m}$ be the (arbitrarily oriented) node--edge incidence of the similarity graph,
$w\in\R^m_{+}$ the edge weights, and $g=B^\top x\in\R^m$ the vector of edge differences
$g_e=x_i-x_j$. With the monotone \emph{edge law}
\begin{equation}\label{eq:plaw}
\rho_p(s)=|s|^{p-1}\mathrm{sign}(s),\qquad \rho_p'(s)=(p-1)|s|^{p-2}\ge0,
\end{equation}
the gradient of \eqref{eq:penergy} is $\nabla E_p(x)=B\,\big(w\odot\rho_p(B^\top x)\big)+M(x-y)$, where
$M=\mathrm{diag}(\mu_i)$ is the diagonal label term---$\mu_i$ is the penalty weight on labeled node
$i\in S$ and $\mu_i=0$ otherwise (uniform $\mu_i\equiv\mu$ in our experiments, as written in
\eqref{eq:quad}--\eqref{eq:penergy}; node-varying $\mu_i$ is handled identically)---and $y$ holds the
targets. Stationarity is the \emph{nonlinear graph Laplacian}
\begin{equation}\label{eq:stat}
B\,\big(w\odot\rho_p(B^\top x)\big)+M(x-y)=0 .
\end{equation}
For $p=2$ this is the SDDM system $(L_w+M)x=My$ with $L_w=B\,\mathrm{diag}(w)\,B^\top$; for $p\ne2$ the
per-edge conductance $w_e\rho_p'(g_e)$ varies with the solution, which is the main difficulty.

\subsection{Grounded-node reduction}
Equation~\eqref{eq:stat} carries a reaction term $M(x-y)$. We remove it with a standard device: introduce one \emph{ground node} $g_0$ and, for each labeled node
$i\in S$, a \emph{linear} edge $i\!-\!g_0$ of constant conductance $\mu_i$. On the augmented incidence
$\hat B\in\R^{(n+1)\times(m+|S|)}$ define the mixed law
\begin{equation}\label{eq:mixlaw}
\hat\rho(g)_e=\begin{cases} w_e\,\rho_p(g_e), & e\ \text{a graph edge},\\[2pt]
\mu_e\,g_e, & e\ \text{a mass edge},\end{cases}
\end{equation}
and the source $b\in\R^{n+1}$ with $b_i=\mu_i y_i$ on $i\in S$, $b_{g_0}=-\sum_{i\in S}\mu_i y_i$, and
zero elsewhere (so $\mathbf 1^\top b=0$). Then \eqref{eq:stat} is exactly the source-form nonlinear
Laplacian flow
\begin{equation}\label{eq:flow}
\hat B\,\hat\rho(\hat B^\top x)=b ,
\end{equation}
solved on the connected augmented graph up to the constant gauge; fixing $x_{g_0}=0$ recovers the SSL
scores. This is the NLF master equation~\cite[Eq.(1.1)]{nlf}, so the NLF solver directly applies.

\begin{proposition}[Reduction correctness]\label{prop:reduce}
Let $x^\star\in\R^{n+1}$ solve \eqref{eq:flow} with $x^\star_{g_0}=0$. Then $x^\star_{1:n}$ is a
stationary point of \eqref{eq:penergy}, i.e.\ it solves \eqref{eq:stat}. Conversely, any solution of
\eqref{eq:stat}, extended by $x_{g_0}=0$, solves \eqref{eq:flow}.
\end{proposition}
\begin{proof}
Orient each mass edge $e$ from the labeled node $i$ to the ground node $g_0$, so
$\hat B_{i,e}=-1$, $\hat B_{g_0,e}=+1$, and its edge difference is
$(\hat B^\top x)_e=x_{g_0}-x_i=-x_i$. Its linear flow is
$\hat\rho(g)_e=\mu_i(x_{g_0}-x_i)=-\mu_i x_i$, so the mass edge contributes
$\hat B_{i,e}\,\hat\rho(g)_e=(-1)(-\mu_i x_i)=+\mu_i x_i$ to row $i$ of $\hat B\hat\rho(\hat B^\top x)$.
Hence row $i$ of \eqref{eq:flow} reads $B(w\odot\rho_p(B^\top x))_i+\mu_i x_i=b_i=\mu_i y_i$, i.e.\
$B(w\odot\rho_p(B^\top x))_i+\mu_i(x_i-y_i)=0$, which is row $i$ of \eqref{eq:stat}. Rows
$i\notin S$ carry no mass edge and match \eqref{eq:stat} directly. The ground row is minus the sum of
the others (both sides sum to zero, using $\mathbf 1^\top b=0$), hence redundant and consistent.
This proves \eqref{eq:flow}$\Rightarrow$\eqref{eq:stat}. Conversely, any solution of \eqref{eq:stat}
extended by $x_{g_0}=0$ satisfies \eqref{eq:flow} by the same computation reversed.

\emph{Existence and uniqueness.} These follow from convexity: \eqref{eq:flow} is $\nabla E=0$ for the augmented
energy $E(x)=\sum_e \Phi_e((\hat B^\top x)_e)-b^\top x$ with $\Phi_e'=\hat\rho_e$; each $\Phi_e$ is
strictly convex (as $\hat\rho_e$ is strictly increasing for $p>1$, and linear mass edges have
$\Phi_e=\tfrac{\mu_e}{2}s^2$), so $E$ is strictly convex and coercive on the gauge quotient
$\{\mathbf 1^\top x=0\}$. The augmented graph is connected: we take the similarity graph to be its
largest connected component (as in all experiments), and each labeled node is joined to $g_0$ by a mass
edge, so $g_0$ links every component containing a label; a component with no label carries no label
information and is dropped. Strict convexity plus coercivity on the connected graph's quotient gives a
unique minimizer, fixed by $x_{g_0}=0$.
\end{proof}

\section{Solver: chord-Newton, near-linear inner solve, continuation in $p$}\label{sec:cont}
\subsection{Algorithm}
We solve \eqref{eq:flow} by a damped inexact chord-Newton iteration (the machinery
of~\cite{nlf}, restated here so the paper is self-contained; the only SSL-specific inputs are the edge
law~\eqref{eq:mixlaw} and source $b$). At iterate $x$ with edge differences $g=\hat B^\top x$ and
residual $r=\hat B\hat\rho(g)-b=\nabla E(x)$, the frozen linearization is the weighted graph Laplacian
$J=\hat B\,\mathrm{diag}(\hat\rho'(g))\,\hat B^\top$ (the Hessian of $E$); the step $\delta$ solves
$J\delta=-r$ inexactly and $x\leftarrow x+\tau\delta$ with $\tau$ from an Armijo line search on the
energy $E$. The inner solve is delegated to a near-linear Laplacian engine---AC~\cite{kyng2016} by default, LAMG$+$~\cite{livne2012,lamgplus} interchangeable---whose setup is built
once and \emph{lazily refreshed}, rebuilt only when the observed residual reduction degrades, so a whole
nonlinear solve costs a small number of near-linear Laplacian solves. Because $E$ is strictly convex and
coercive (Proposition~\ref{prop:reduce}) and $J\succ0$, each $\delta$ is a descent direction and the
Armijo iteration converges globally to the unique equilibrium provided the inner solves reduce the
linear residual below a forcing threshold, so that $\delta$ stays a descent direction---the standard
inexact-Newton condition of Dembo--Eisenstat--Steihaug~\cite{des1982}, which a fixed inner tolerance
(here $10^{-6}$, well below the outer tolerance) meets on the SPD systems $J\succ0$ here. We prove this
global convergence self-contained in Appendix~\ref{app:conv}, from the strict convexity of
Proposition~\ref{prop:reduce} alone; it does not rely on~\cite{nlf}.
On most graphs this takes a handful of steps, but on stiff, hub-dominated graphs at large $p$ the
conductance floor turns $J$ into a heavily \emph{modified} Hessian (most edges are near-flat, so
$\hat\rho'(g)\!\to\!0$ is replaced by the floor) and the iteration contracts only \emph{linearly}, with
rate approaching $1$---a slow crawl, not a failure. We cancel this last mode with the NLF solver's
guarded \emph{Anderson acceleration} of the outer iteration: the last few accepted iterates
are mixed by a small least-squares solve, and the mix is accepted only when it does not increase the
energy $E$ relative to the plain step. The guard
leaves the fixed point and the monotone global convergence untouched---it can only lower the step
count. With it, the outer step count stays bounded and we observe
convergence on all $228$ corpus graphs in at most $50$ steps (\S\ref{sec:scaling}).

\subsection{Conductance floor ($\varepsilon$-regularization)}
For $p>2$ the conductance $w_e\rho_p'(g_e)=w_e(p-1)|g_e|^{p-2}$ \emph{vanishes} as $g_e\to0$, so on the
flat regions that dominate a propagated solution the frozen Laplacian loses those edges and becomes
ill-conditioned. We floor the conductance at $\varepsilon\,w_{\max}$ with $\varepsilon=10^{-6}$
throughout---an $\varepsilon$-regularization, the graph analogue of the lagged-diffusivity /
relaxed-reweighting device standard for the $p$-Laplacian and total variation. Crucially, the floor
enters \emph{only} the frozen Jacobian $J$ (the linear model for the Newton step), not the residual
$r=\nabla E$ or the energy $E$, both of which use the exact $p$-law. It therefore changes the step
\emph{direction} but not the equation being solved: the fixed point $\nabla E=0$---the SSL
solution---is \emph{exactly} independent of $\varepsilon$; the floor is a preconditioning choice, not a
change of model. Consistent with this, the reported accuracies are unchanged (to the stated precision)
across $\varepsilon\in[10^{-8},10^{-6}]$. (For $p<2$ the same floor caps the opposite, blow-up degeneracy at $g_e\to0$.)

\subsection{Continuation in $p$}
Jumping directly to a large $p$ makes the frozen linearization a poor global model and stalls the
iteration---indeed primal IRLS is known to diverge for $p\ge3$~\cite{storn2026}. We instead
\emph{continue} in $p$: starting from the $p=2$ solution (a single near-linear Laplacian solve), we solve a short
increasing schedule $2=p_0<p_1<\cdots<p_K=p$, warm-starting each solve from the previous and sharing one
lazily refreshed hierarchy across the whole chain (Algorithm~\ref{alg:cont}). We use the geometric
schedule $p_{k+1}=\min(p,\,1.5\,p_k+0.5)$; the multiplier is chosen just large enough that each new
stage lands inside the basin of the previous one (a smaller step wastes solves, a larger one risks
leaving the basin and stalling), and $p=2\to3$ falls out as a single nonlinear stage after the linear
$p=2$ solve. The step count is insensitive to the exact multiplier over $[1.3,1.7]$; we did not tune it
per graph. Each stage begins inside the basin of the previous, cheaper stage, so the chord-Newton
count per stage is small and grows only mildly with size (Table~\ref{tab:family}); the total grows
mildly with the target $p$ as stages are added.

\begin{algorithm}
\caption{$p$-Laplacian SSL by continuation in $p$ (NLF inner solver).}\label{alg:cont}
\begin{algorithmic}[1]
\State \textbf{input:} graph $(B,w)$, labels $(S,y)$, clamp $\mu$, target $p$, schedule $2=p_0<\dots<p_K=p$
\State build augmented $(\hat B, b)$ and mass conductances $\mu$ \Comment{grounded-node reduction, \S\ref{sec:method}}
\State $x\gets 0$;\; hierarchy $H\gets\varnothing$
\For{$k=0,1,\dots,K$}
  \State $x\gets$ chord-Newton solve of $\hat B\hat\rho_{p_k}(\hat B^\top x)=b$, warm-started from $x$,
         inner solves via the near-linear engine on $H$ (lazy refresh) \Comment{NLF \texttt{newton\_flow!}}
\EndFor
\State \Return scores $x_{1:n}-x_{g_0}$
\end{algorithmic}
\end{algorithm}

\subsection{Complexity}
The total work is the product of three factors: (1) number of continuation stages; (2) the chord-Newton
steps per stage; and (3) the cost of one inner solve. 

The first factor is $\mathcal{O}(\log p)$ by the fixed schedule construction. The last is near-linear: the frozen Hessian
$J=\hat B\,\mathrm{diag}(\hat\rho'(g))\,\hat B^\top$ (with the conductance floor) is a symmetric
diagonally dominant $M$-matrix. AC solves such a system to fixed relative accuracy
in $\mathcal{O}(m\log^{3}n)$ expected time~\cite{kyng2016}, while LAMG+~\cite{lamgplus} is empirically $\mathcal{O}(m)$. The middle factor---the per-stage
step count---has no clean a priori bound on general graphs, but with the guarded Anderson acceleration
(\S\ref{sec:cont}) it stays flat in practice ($\le16$ on the controlled families of
Table~\ref{tab:family}, $\le50$ on the stiffest corpus graph) and, crucially, does \emph{not} grow with
$m$. We emphasize that this middle factor is an \emph{empirical} $\mathcal{O}(1)$, not a proved bound;
the near-linear per-solve cost (first and third factors), by contrast, is the expected-time guarantee of
AC~\cite{kyng2016}. Conditional on the observed flat step count, the whole solve is thus near-linear:
$\mathcal{O}(m\log^{3}n\,\log p)$ expected with AC, and empirically $\mathcal{O}(m\log p)$
with the LAMG+ inner engine. We use AC as the faster practical default and report
LAMG+ to demonstrate an $\Om$ inner solve is feasible.

Why, then, is the measured corpus wall-clock $m^{1.19}$ rather than $m^1$? Not because of the
outer iteration step count and not because of the $p$-Laplacian reweighting: a
single $p=2$ solve with one clean Laplacian and no reweighting per graph already fits $m^{1.14}$ pooled over the same heterogeneous corpus (not a fixed-family fit), and the full $p=3$ continuation adds only ${\approx}0.05$ to that pooled exponent. The mild superlinearity lives entirely in the inner solve's per-edge cost due to a pooling artifact: the corpus mixes graph
classes whose per-edge constant differs several-fold (Simpson's paradox), and the larger, harder
classes populate the high-$m$ end, steepening a single power-law fit. Neither reflects super-linear
$m$-scaling of the algorithm on any fixed graph type (Table~\ref{tab:expdecomp}).

\subsection{Multiclass and hierarchy reuse across labels}
For $C$ classes we solve \eqref{eq:flow} once per class with one-vs-rest targets $y_i=\pm1$ and take the
argmax of the $C$ score vectors. Crucially, the augmented graph $\hat B$ is \emph{identical} across the
$C$ one-vs-rest problems---same graph, same labeled set; only the source $b$ (the $\pm1$ targets)
differs---so the inner setup is amortized over all $C$ solves. At the $p=2$ stage the operator is
literally identical across classes and a single hierarchy is shared exactly; at nonlinear stages the
frozen hierarchy is reused across classes as a preconditioner with the same lazy refresh (the
multicommodity reuse of~\cite{nlf}). Writing $T_{\mathrm{solve}}$ and $T_{\mathrm{setup}}$ for the time
of one inner solve and one hierarchy build---each near-linear, $\Om$---the total time is therefore
\begin{equation}\label{eq:cost}
\mathcal{O}\big(C\,T_{\mathrm{solve}}+T_{\mathrm{setup}}\big),
\end{equation}
one amortized setup shared across all $C$ classes rather than $C$ independent setups.

\section{The low-label degeneracy fix}\label{sec:accuracy}
We first confirm the solver realizes the statistical effect that motivates $p>2$. Following the standard
protocol~\cite{calder2020} we build a symmetric $10$-nearest-neighbor Gaussian similarity graph and
label $r$ points per class, sweeping $r$ and $p$ and reporting test accuracy on the unlabeled nodes
(mean over ten label draws). We use the canonical SuiteSparse ML\_Graph MNIST $10$-NN graph
($n=10{,}000$, $m=72{,}800$).\footnote{Node labels align with the canonical MNIST test set: a $p=2$
probe at ten labels/class scores $87.5\%$, confirming the graph construction is sound.}

\begin{table}[t]\centering
\caption{MNIST $10$-NN graph, test accuracy (\%, mean$\pm$std over ten label draws) by labels/class
(rows) and exponent $p$ (columns). The quadratic $p=2$ estimate collapses at one label/class
($36\%$) and the nonlinear $p=3$ estimate recovers it ($64\%$); the advantage shrinks as labels accumulate and is a statistical tie by five labels/class,
where $p=2$ is already well-posed. The non-monotone variation across $p\ge5$ within a row is within the
reported trial-to-trial std.}
\label{tab:acc}
{\small\setlength{\tabcolsep}{4pt}%
\begin{tabular}{c|cccccc}
\toprule
labels/class & $p{=}2$ & $p{=}3$ & $p{=}5$ & $p{=}8$ & $p{=}15$ & $p{=}30$\\
\midrule
1 & $35.7{\pm}10.4$ & $\mathbf{64.0}{\pm}8.2$ & $63.2{\pm}5.9$ & $58.8{\pm}7.2$ & $60.9{\pm}6.0$ & $58.5{\pm}4.8$\\
2 & $56.4{\pm}8.0$ & $\mathbf{75.6}{\pm}4.4$ & $72.1{\pm}4.2$ & $68.1{\pm}3.5$ & $69.6{\pm}4.2$ & $70.3{\pm}4.3$\\
3 & $73.6{\pm}8.7$ & $\mathbf{79.5}{\pm}3.7$ & $76.5{\pm}2.8$ & $73.1{\pm}3.7$ & $73.0{\pm}3.2$ & $75.7{\pm}2.5$\\
5 & $83.6{\pm}3.4$ & $\mathbf{83.7}{\pm}2.4$ & $81.2{\pm}2.6$ & $78.7{\pm}2.8$ & $78.2{\pm}2.8$ & $77.7{\pm}2.3$\\
\bottomrule
\end{tabular}}
\end{table}

Table~\ref{tab:acc} and Figure~\ref{fig:acc} show the picture the degeneracy theory anticipates: the
$28$-point $p{=}2\!\to\!3$ gain at one label/class is statistically unambiguous despite the large
per-draw spread (a Welch two-sample $t$-test over the ten draws gives $t\approx6.8$, $p<10^{-4}$), and
it closes to a tie by five labels/class. We discuss the numbers in the caption and turn to the one
subtlety that matters.
The empirical optimum is a moderate $p$ (here $p\approx3$), not the theoretical $p>d$, where $d$ is the intrinsic dimension of the data manifold (in case of MNIST, $\approx13$~\cite{hein2005,pope2021}). That's because the $p>d$ results of
\cite{slepcev2019,calder2019} are \emph{asymptotic} statements about the $n\to\infty$ continuum limit at
fixed labels, whereas Table~\ref{tab:acc} reports finite-$n$ ($n=10^4$) accuracy. At finite $n$ the
$p=2$ estimate is not yet fully collapsed, and a moderate $p$ already lifts it out of the near-degenerate
regime; pushing $p$ toward and beyond the intrinsic $d$ trades the well-posedness benefit against
increased smoothing bias, so the accuracy optimum sits well below the asymptotic threshold. In sum, the
theory motivates \emph{why non-quadratic $p$ helps at low labels}; it does not pin the accuracy-optimal
$p$, which is set by finite-$n$ bias--variance. This actually helps our solver: the useful $p$ is
moderate and reached by a small number of continuation steps.

\subsection{Independent validation}
The effect is not an artifact of our graph construction. Calder's \texttt{GraphLearning} package~\cite{calder2020}, using a different similarity graph (variational-autoencoder embeddings, Gaussian
$10$-NN kernel), reproduces the same collapse-and-recovery: at one label per class, Laplace ($p=2$)
learning scores $22.4\%$ on MNIST and $16.9\%$ on FashionMNIST, while its $p$-Laplace ($p=3$) solver
recovers to $71.6\%$ and $55.8\%$ respectively. Both the degeneracy and its non-quadratic fix are
solver- and graph-independent.

\subsection{The fix, demonstrated at web scale}
The value of a scalable solver is that the same statistical fix can now be run where the graph is large
and the labels few---precisely the industrial regime. On \texttt{ogbn-products}~\cite{ogb} (Amazon
co-purchasing, largest component $2.39\times10^6$ nodes, $6.18\times10^7$ edges, $47$ classes) we solve
a full $p=3$ SSL propagation in $527$ seconds on a laptop, and compare $p=2$ against $p=3$ on the three
largest classes at five labels per class, using \emph{graph structure alone} (no node features). The
feature-free setting is deliberate and is the correct one for our question: node features would let a
classifier partly bypass propagation, confounding the pure graph-SSL effect we test, whereas the
low-label industrial regime we target (few-shot labeling on a large interaction graph) is exactly where
features are scarce or absent and label propagation carries the signal. Quadratic propagation scores
$21.5\%$---even below the $33\%$ three-class random line, while $p=3$ recovers to $37.1\%$ ($+15.6$
points). This demonstrates the fix using our solver at a scale where the direct-factorization incumbents
cannot form the factor at all (\S\ref{sec:scaling}); we make no claim to beat feature-based GNNs here,
only that nonlinear propagation is both necessary (quadratic fails) and now feasible at this size.

\begin{figure}[t]\centering
\includegraphics[width=\textwidth]{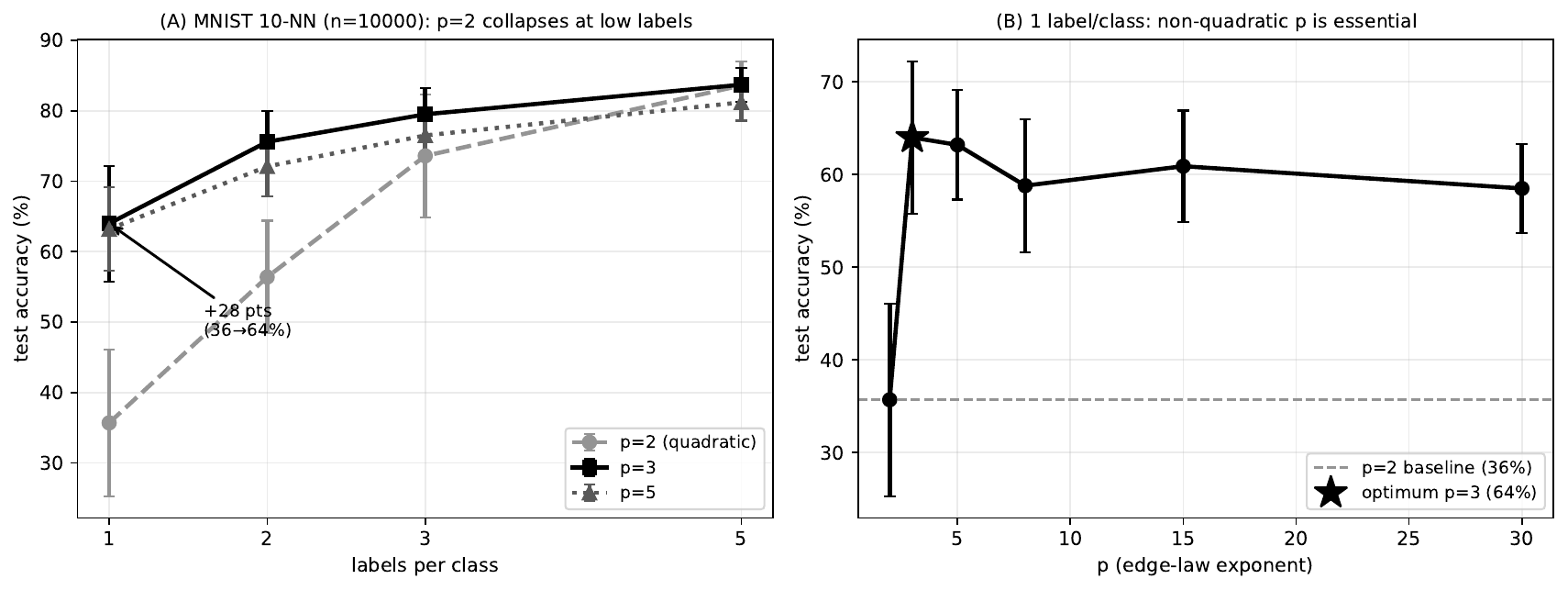}
\caption{Low-label degeneracy fix on the MNIST $10$-NN graph. \textbf{(A)} Test accuracy vs labels per
class: the quadratic $p=2$ estimate (circles, dashed) collapses at one label/class where $p=3$ (squares, solid) recovers.
\textbf{(B)} Accuracy vs $p$ at one label/class: a moderate non-quadratic $p$ is essential; the optimum
is $p\approx3$.}\label{fig:acc}
\end{figure}

\section{Near-linear scaling}\label{sec:scaling}
We run over a corpus of 228 real graphs from the SuiteSparse collection. The corpus is all graphs from
the genuine data-similarity families (DIMACS10, SNAP, Newman, Pajek, ML\_Graph, Gleich, Arenas,
Barabasi, LAW, Butterfly) with $10^3$--$1.7\times10^7$ edges, excluding census (\texttt{*2010}) and VLSI
(\texttt{vsp\_*}) matrices, which are not similarity graphs; we take each graph's largest connected
component with symmetrized weights.\footnote{The exact graph list is pinned in the repository
(\texttt{scripts/exp\_corpus\_final.jl}).} On each we plant a handful of random $\pm1$ seeds and solve
\eqref{eq:penergy} at the target $p=3$ by continuation from $p=2$, with the near-linear inner engine
(AC); we repeat the sweep with LAMG$+$ to confirm the two are interchangeable, and
compare against the incumbent inner solve---a fresh sparse Cholesky factorization at every Newton
linearization. All $228$ converge (\S\ref{sec:cont}); scaling fits use the $m\ge10^3$ graphs.

\subsection{Controlled scaling: $\Om$ at fixed graph type}
The right way to isolate size scaling is to hold the graph type fixed and vary the size. We
sweep two \emph{families}---Delaunay triangulations ($2^{14}$--$2^{20}$ nodes) and random geometric
graphs ($2^{15}$--$2^{19}$ nodes)---each over $\ge5$ sizes (Table~\ref{tab:family}). With the LAMG$+$
engine the wall-clock scales as $m^{1.02}$ (Delaunay, $R^2{=}0.999$) and $m^{0.96}$ (RGG,
$R^2{=}0.997$), with $10$--$14$ chord-Newton steps at every size. AC is marginally steeper ($m^{1.10}$, $m^{1.18}$). This only
certifies $\Om$ at fixed graph type. Fitting one power law across the heterogeneous corpus below
gives a larger slope ($m^{1.19}$), but that is an artifact of aggregating distinct graph classes
with different per-solve constants (a larger, harder class dominates the high-$m$ end and steepens
the aggregate line)---not super-linear scaling of the algorithm on any fixed class. We verified this is
not simply a density effect: controlling for edge density in a multiple regression leaves the corpus size
coefficient unchanged, whereas restricting to any single family recovers $m^{0.96}$--$m^{1.02}$.

\begin{table}[t]\centering
\caption{Controlled size-scaled families (same graph type, doubling size), $p=3$ SSL wall-clock (s).
LAMG$+$ scales as $m^{0.96}$--$m^{1.02}$ (per family; $R^2\ge0.997$)---genuine $\Om$; the chord-Newton step count is flat.
The fitted exponents (LAMG$+$): Delaunay $m^{1.02}$ ($R^2{=}0.999$, $7$ points), RGG $m^{0.96}$
($R^2{=}0.997$, $5$ points); AC is marginally steeper ($m^{1.10}$, $m^{1.18}$).
Times are from the dedicated warm+solo controlled-family run
(\texttt{results/family\_scaling.txt}); at sub-second sizes the per-run warmup constant differs mildly
from the corpus run of Fig.~\ref{fig:scaling} without affecting the fitted exponents.}
\label{tab:family}
\begin{tabular}{l r r r r}
\toprule
family & $m$ (edges) & approx.\ Chol.\ (s) & LAMG$+$ (s) & steps\\
\midrule
Delaunay $n{=}14$ & $49{,}122$    & $0.22$ & $0.52$ & $10$\\
Delaunay $n{=}15$ & $98{,}274$    & $0.44$ & $1.09$ & $11$\\
Delaunay $n{=}16$ & $196{,}575$   & $0.98$ & $2.33$ & $12$\\
Delaunay $n{=}17$ & $393{,}176$   & $2.09$ & $4.34$ & $12$\\
Delaunay $n{=}18$ & $786{,}396$   & $4.00$ & $8.04$ & $11$\\
Delaunay $n{=}19$ & $1{,}572{,}823$ & $10.3$ & $18.5$ & $14$\\
Delaunay $n{=}20$ & $3{,}145{,}686$ & $20.5$ & $37.8$ & $13$\\
\midrule
RGG $n{=}15$ & $160{,}235$   & $0.54$ & $1.09$ & $13$\\
RGG $n{=}16$ & $342{,}127$   & $1.16$ & $2.21$ & $10$\\
RGG $n{=}17$ & $728{,}750$   & $3.41$ & $5.19$ & $13$\\
RGG $n{=}18$ & $1{,}547{,}283$ & $7.89$ & $10.3$ & $13$\\
RGG $n{=}19$ & $3{,}269{,}760$ & $17.5$ & $18.7$ & $13$\\
\bottomrule
\end{tabular}
\end{table}

\begin{figure}[t]\centering
\includegraphics[width=\textwidth]{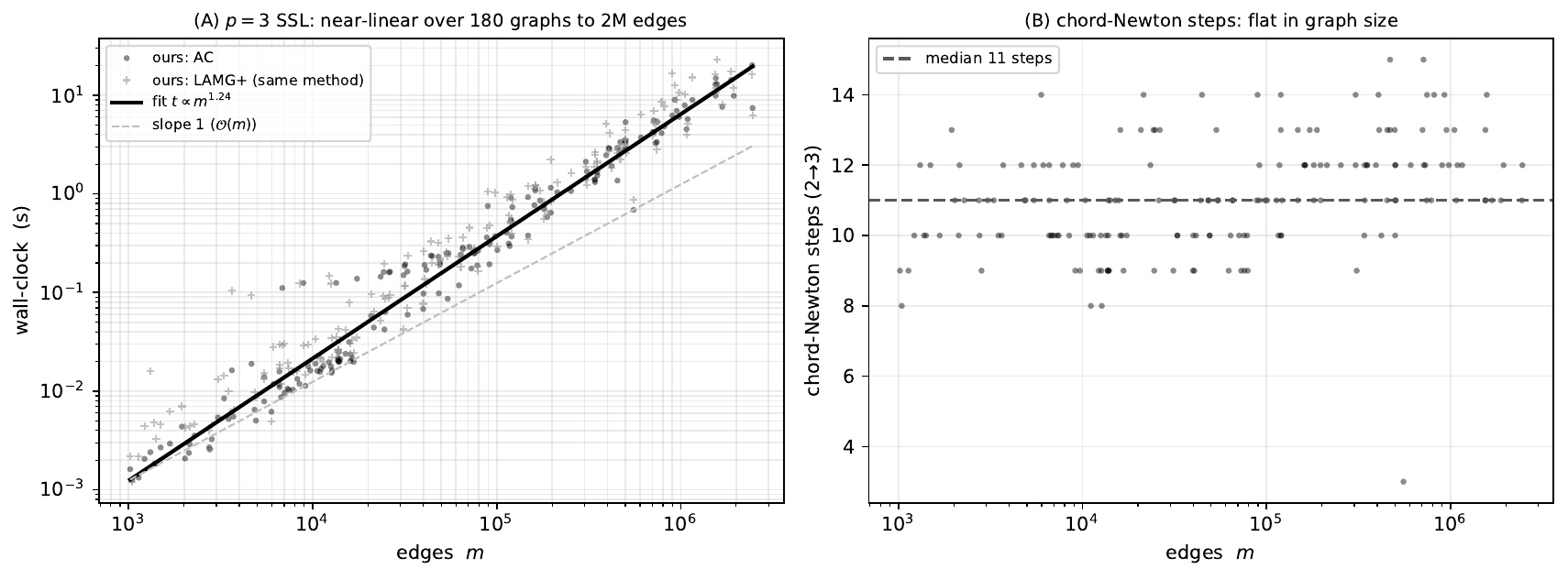}
\caption{\textbf{(A)} Our $p=3$ SSL wall-clock vs edges $m$, log--log, spanning nearly four decades to
the $6.8\times10^7$-edge \texttt{soc-LiveJournal1} (star). AC (filled circles) and LAMG$+$
($+$ markers) are near-linear ($m^{1.19}$ over the corpus; $m^{0.96}$--$m^{1.02}$ on the controlled families,
Table~\ref{tab:family}); the incumbent IRLS+direct (squares) turns superlinear.
\textbf{(B)} Chord-Newton step count (circles) is flat in $m$; dashed line marks the median.}\label{fig:scaling}
\end{figure}

\begin{figure}[t]\centering
\includegraphics[width=0.82\textwidth]{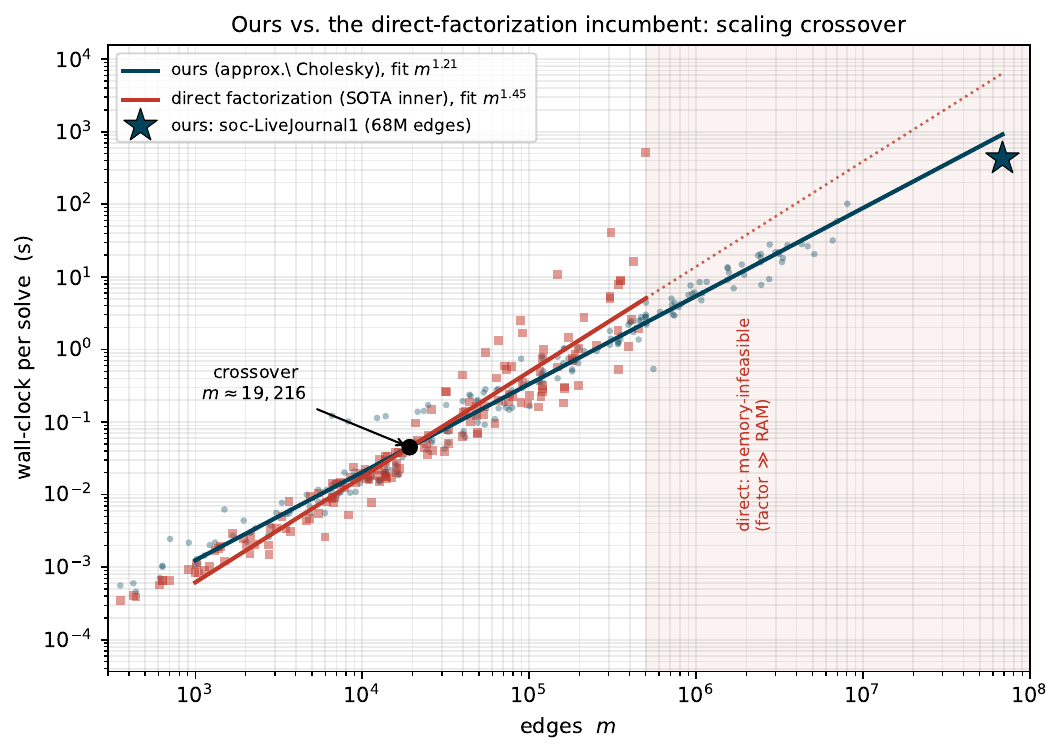}
\caption{Like-for-like total wall-clock per graph, our default (AC, teal) vs the
direct-factorization incumbent (red), solving the same problem to the same tolerance, over the corpus. The
fitted lines cross at $m\approx1.9\times10^4$: direct is cheaper on small/well-separated graphs but its
cost grows super-linearly with fill ($m^{1.45}$ vs our $m^{1.19}$). The wide vertical spread of the red
points at fixed $m$ is graph-type dependence (geometric cheap, irregular explosive; cf.\
Table~\ref{tab:crossover}). Past the shaded region the direct factor exceeds RAM and is infeasible,
while ours reaches the $68$M-edge giant (star).}\label{fig:crossover}
\end{figure}

Over the full heterogeneous corpus (Figure~\ref{fig:scaling}, spanning nearly four decades of edge
count) the wall-clock fits $t\propto m^{1.19}$ (least-squares
over the converged graphs with $m\ge10^3$). The chord-Newton step count is flat in $m$---median
$11$, $\le50$ over the corpus, fitted $\text{steps}\propto m^{0.03}$---so the outer iteration is
$\mathcal O(1)$ and the residual excess over $m^1$ comes from the aforementioned near-linear inner solve's per-edge cost increase.
With the guarded Anderson acceleration (\S\ref{sec:cont}) convergence holds on \textbf{all $228$ corpus
graphs}: it rescues the one previously stalling case (an extreme power-law autonomous-system graph that
without acceleration crawls at linear rate ${\approx}0.99$ and does not converge within the step budget,
and with it converges in $50$ steps) without altering any other solve: by construction the energy guard
accepts a mixed step only when the energy does not increase, so acceleration can only lower the step
count, never raise it, and on the remaining $227$ graphs---which already converge in a median of $11$
steps---it changes nothing. Table~\ref{tab:expdecomp} decomposes the
measured exponent by factor and engine.

\begin{table}[t]\centering\small
\caption{Where the wall-clock exponent comes from. The outer iteration is $\mathcal{O}(1)$ (flat step
count) and the $p$-Laplacian reweighting is minor (a single clean $p=2$ solve over the same corpus
already fits $m^{1.14}$; the full $p=3$ continuation adds only ${\approx}0.05$). The residual excess over
$m^1$ is the near-linear inner solve's per-edge cost: LAMG+ is empirically $\Om$ at fixed graph type,
while AC carries its provable $\mathcal{O}(m\log^{3}n)$ factor. Within a fixed graph
family the scaling is (near-)linear; the larger pooled-corpus exponent is a mixed-class aggregate, not
super-linear scaling of the algorithm.}\label{tab:expdecomp}
\begin{tabular}{lll}
\toprule
factor & scope & fitted exponent\\
\midrule
outer chord-Newton steps      & vs.\ $m$, corpus        & $m^{0.03}$ \ (flat, $\mathcal{O}(1)$)\\
$p$-Laplacian reweighting     & $p{=}2\!\to\!p{=}3$, corpus & $m^{1.14}\!\to\!m^{1.19}$ \ ($+0.05$)\\
inner solve, LAMG+            & per family              & $m^{0.96}$--$m^{1.02}$ \ (empirical $\Om$)\\
inner solve, AC               & per family              & $m^{1.10}$--$m^{1.18}$ \ ($\mathcal{O}(m\log^{3}n)$)\\
pooled corpus fit             & 228 graphs              & $m^{1.19}$ \ (mixed-class)\\
\bottomrule
\end{tabular}
\end{table}

\subsection{Inner engine ablation}
The two engines give near-identical iteration counts (mean per-graph difference $\approx\!1.6$ steps) and
interchangeable solutions, so the choice is purely one of inner-solver constants. Empirically AC is the faster practical choice:
it is $1.6\times$ faster than LAMG$+$ in median wall-clock and quicker on $86\%$ of the corpus, holding
a $1.3\times$ edge even above $10^6$ edges. LAMG$+$, on the other hand, has the cleaner exponent
($m^{0.96}$--$m^{1.02}$ per family, mean $m^{0.99}$, vs AC's $m^{1.10}$--$m^{1.18}$), so its larger constant is eventually amortized and
it would overtake at some scale beyond those we test. We therefore default to AC for
practice; the method itself is agnostic to the choice of near-linear Laplacian engine.

\subsection{Comparison with the incumbent solvers}\label{sec:sota}
The state-of-the-art variational $p$-Laplacian solvers reduce the nonlinear problem to a sequence of
reweighted Laplacian solves; their released implementations differ in how they perform that inner solve.
The IRLS and dual-IRLS variants (Flores--Calder--Lerman's \texttt{irls\_solve}, Storn) use a
direct sparse Cholesky factorization; FCL's Newton-with-homotopy---the method the same
authors recommend as the practical incumbent---uses an incomplete-Cholesky-PCG. We
compare against both: the direct variant (a same-code inner-solver swap that exposes its memory
wall, \S\ref{sec:crossover}), and FCL's released Newton code head-to-head in \S\ref{sec:fcl}.

\subsubsection{Scaling crossover}\label{sec:crossover}
To fairly compare against the direct variant, we fix the
problem, the outer continuation, the $p$-schedule, and the convergence tolerance ($10^{-6}$), and swap
only the inner Laplacian solve to an exact sparse Cholesky. Both variants converge to the same
SSL solution (the engines are interchangeable to $\sim\!10^{-10}$, \S\ref{sec:cont}); the realized
Newton-step counts differ by a few (mean per-graph difference $\approx\!3$ steps) because the near-linear inner solve is
inexact, so we compare total wall-clock to solve the same problem to the same accuracy (not a per-step cost). Every solve is run solo and warm-timed.

Figure~\ref{fig:crossover} plots both against edge count over the corpus. The direct factorization is
faster on small and well-separated graphs---its constant is a highly tuned library---but its per-solve
cost grows super-linearly with fill ($m^{1.45}$ fitted) whereas ours stays near-linear ($m^{1.19}$), so
the two fitted lines \emph{cross} at $m\approx1.9\times10^4$: beyond a few $\times10^4$ edges our
near-linear engine is faster on average, and past the memory wall the direct factor no longer fits
in RAM at all while ours runs to the $68$M-edge giant. The wide vertical spread of the direct points at
fixed $m$ is the real content: direct's cost is set by graph type, not size.

Table~\ref{tab:crossover} shows this in detail. At a fixed size ($\sim\!2$--$5\times10^5$ edges) the
direct solve swings over three orders of magnitude in wall-clock---$0.3$\,s on a road or Delaunay mesh
(good separators, near-linear fill) up to $519$\,s on a scale-free graph with a $568$M-nonzero factor
($\approx9$\,GB)---while ours stays a few seconds regardless. On geometric graphs direct wins; on
exactly the irregular kNN/social/scale-free graphs that SSL is built for, ours is already
$1.6$--$173\times$ faster by a few hundred thousand edges, and then the factorization runs out of memory. We do not claim to beat a tuned direct solve where it fits; but we show it does not work on the graphs that matter, and that the crossover is already at $m\sim10^4$. The root cause is fill: on the irregular graphs SSL uses the Cholesky factor grows as $\sim\!\mathrm{nnz}^{1.57}$ within a graph family and reaches $10$--$280\times$ the graph's nonzeros (Table~\ref{tab:fill}), versus our $\mathcal{O}(m)$ hierarchy.

\begin{table}[t]\centering
\caption{Total wall-clock to solve the \emph{same} SSL problem to the same tolerance: the
direct-factorization incumbent vs our default (AC), same outer continuation,
warm-timed solo. Direct's cost is governed by fill: cheap on graphs with good vertex separators,
catastrophic on the irregular kNN/social/scale-free graphs SSL uses, and ultimately infeasible
(factor $\gg$ RAM). Speedup ${>}1$ favors ours. The largest direct solve we time is
\texttt{prefAttachment} ($519$\,s, $568$M-nonzero factor, ${\approx}9$\,GB). $^{\dagger}$Factor
\emph{measured} by an actual sparse Cholesky (Table~\ref{tab:fill}): $508$M nonzeros, ${\approx}8.1$\,GB;
we report no solo wall-clock because the peak factorization working set (several times the final factor)
makes a clean solo timing impractical, so ``ours'' already wins outright. $^{\ddagger}$Factor not formed;
its size is \emph{extrapolated} via the within-family fill law
$\mathrm{nnz}(\text{factor})\!\sim\!\mathrm{nnz}^{1.57}$ (fitted on the graphs that do factorize) to
terabyte scale---genuinely infeasible.}
\label{tab:crossover}
{\small\setlength{\tabcolsep}{4pt}%
\begin{tabular}{l l r r r r}
\toprule
graph & type & $m$ & direct (s) & ours (s) & speedup\\
\midrule
\texttt{usroads}       & road mesh (good sep.) & $162$K & $0.30$ & $1.02$ & $0.3\times$\\
\texttt{delaunay\_n16} & mesh (good sep.)      & $197$K & $0.32$ & $1.06$ & $0.3\times$\\
\midrule
\texttt{ca-AstroPh}    & collaboration             & $197$K & $1.03$ & $0.65$ & $1.6\times$\\
\texttt{loc-Brightkite}& social                    & $213$K & $2.75$ & $0.95$ & $2.9\times$\\
\texttt{k49\_10NN}     & kNN similarity            & $309$K & $41.2$ & $1.41$ & $29\times$\\
\texttt{prefAttachment}& scale-free                & $500$K & $519$  & $3.01$ & $\mathbf{173\times}$\\
\midrule
\texttt{amazon0302}    & co-purchase               & $0.9$M & $8.1$\,GB factor$^{\dagger}$ & $5.61$ & ---\\
\texttt{soc-LiveJournal1} & social                 & $68$M  & \emph{infeas.}$^{\ddagger}$    & $429$  & $\infty$\\
\bottomrule
\end{tabular}}
\end{table}

\subsubsection{Head-to-head against the released FCL solver}\label{sec:fcl}
The sharpest test is to run the incumbents' \emph{own} released code. We took the FCL variational
Newton-with-homotopy solver~\cite{fcl2022} (\texttt{nt\_solve\_newton}, MATLAB/Octave, ichol-
PCG inner solve, continuation in $p$) and ran it, unmodified, against ours on \emph{identical} kNN
similarity graphs---the exact setting FCL targets---with the same labeled set and $p$-target. Because
both solve the same variational operator, they reach the same classification: binary accuracies are
identical (to three figures) on every case except Fashion at $p{=}8$, where they differ by one node in
$2{,}500$; the solution fields agree to $\|\cdot\|$-relative $10^{-3}$ at $p{=}3,5$ (and to $\sim\!2\text{--}3\times10^{-2}$ at $p{=}8$, where the near-flat field is not tightly pinned by a $10^{-8}$ residual). The wall-clock ratio is thus an honest solver-to-solver comparison at matched output.

In Table~\ref{tab:fcl}, our solver is faster in every case, and the margin
grows with graph size and with $p$: from $1.5\times$ on MNIST ($73$K edges) to
$14\times$ on \texttt{k49} ($p{=}5$, $309$K edges). One caveat on the absolute factors: ours is in Julia
and FCL's released code is MATLAB/Octave, so a constant part of the ratio is implementation and
language, not algorithm---the $1.5\times$ on small MNIST is within that gray zone and we do not lean on
it. What is \emph{not} a language effect is the \emph{growth} of the margin, and its mechanism is
visible in the iteration counts, which are language-independent: FCL's incomplete-Cholesky
preconditioner degrades on the higher-degree, stiffer operators that larger $p$ and irregular kNN graphs
produce (\texttt{k49} has max degree $551$), so its inner CG iteration count climbs, whereas our
near-linear inner solve holds a near-constant cost. The two solvers take a comparable number of
\emph{outer} Newton steps (columns 6--7).

\begin{table}[t]\centering
\caption{Head-to-head against the \emph{released} FCL Newton solver~\cite{fcl2022} (its own
ichol+CG code, run under Octave) on identical kNN graphs, same labels, same $p$-target reached by
continuation. Both solve the same variational operator to the same classification (agreement column:
binary accuracy where ground truth exists---identical except Fashion $p{=}8$---else relative
$\|u_{\mathrm{FCL}}-u_{\mathrm{ours}}\|$). Median of $3$ warm, solo runs. Ours (AC)
is faster throughout, by a margin that grows with size and $p$.}
\label{tab:fcl}
{\small\setlength{\tabcolsep}{5pt}%
\begin{tabular}{l r r r r r r c}
\toprule
graph ($m$ edges) & $p$ & ours (s) & FCL (s) & speedup & ours it. & FCL it. & agreement\\
\midrule
MNIST $10$-NN ($73$K)   & $3$ & $0.33$ & $0.52$ & $1.6\times$ & $12$ & $11$ & acc $0.978$\\
MNIST $10$-NN ($73$K)   & $5$ & $0.64$ & $0.99$ & $1.5\times$ & $20$ & $23$ & acc $0.954$\\
MNIST $10$-NN ($73$K)   & $8$ & $0.88$ & $1.36$ & $1.5\times$ & $28$ & $29$ & acc $0.934$\\
\midrule
Fashion $10$-NN ($79$K) & $3$ & $0.35$ & $1.37$ & $3.9\times$ & $12$ & $11$ & acc $0.434$\\
Fashion $10$-NN ($79$K) & $5$ & $0.93$ & $2.57$ & $2.8\times$ & $26$ & $24$ & acc $0.388$\\
Fashion $10$-NN ($79$K) & $8$ & $0.82$ & $3.85$ & $4.7\times$ & $30$ & $31$ & acc $0.385$\\
\midrule
\texttt{k49} $10$-NN ($309$K) & $3$ & $1.84$ & $23.5$ & $\mathbf{13\times}$  & $12$ & $12$ & $2.4\times10^{-3}$\\
\texttt{k49} $10$-NN ($309$K) & $5$ & $3.33$ & $47.3$ & $\mathbf{14\times}$  & $22$ & $24$ & $1.0\times10^{-3}$\\
\bottomrule
\end{tabular}}
\end{table}

\subsection{A cross-operator third-party point}
As a complementary data point across a \emph{different} operator, we also time Calder's
\texttt{GraphLearning}~\cite{calder2020}. It sidesteps the memory wall: it solves the game-theoretic
$p$-Laplacian by a matrix-free fixed-point iteration, so it never forms a factor; but it pays for that
in additional iterations. On identical graphs it is $5.5\times$ slower on \texttt{delaunay\_n16}, $23\times$ slower
on \texttt{delaunay\_n18} ($91$\,s vs our $4.0$\,s), and $6.6\times$ slower on \texttt{rgg\_n\_2\_17},
because its $\infty$-Laplacian sweep propagates slowly across high-diameter meshes; on low-diameter
social graphs (\texttt{ca-AstroPh}) it is competitive. This is a cross-\emph{operator} comparison (the
game-theoretic and variational $p$-Laplacians differ), so we read it not as a like-for-like solver race
but as evidence that the two existing routes to scalable graph $p$-Laplacian learning---matrix-free
game-theoretic iteration, or our near-linear variational solve---trade off, and ours is markedly faster
on exactly the geometric graphs where the fixed point stalls.

\begin{table}[t]\centering
\caption{Sparse-Cholesky fill (the direct incumbent's memory) relative to the graph's nonzeros
($\mathrm{nnz}=2m$), for the $p=2$ Laplacian. Near-linear ($3$--$7\times$) on graphs with good
separators; super-linear on the irregular graphs SSL uses, making the direct factorization
memory-infeasible at scale. Our hierarchy memory is $\mathcal{O}(m)$.}\label{tab:fill}
\begin{tabular}{l r r}
\toprule
graph & nnz & fill\,/\,nnz\\
\midrule
Delaunay $n{=}18$ (geometric) & $1.6$M & $7.0$\\
RGG $n{=}18$ (geometric)      & $3.1$M & $3.8$\\
\midrule
MNIST $10$-NN (kNN similarity) & $146$K & $42.2$\\
\texttt{ca-AstroPh} (collaboration) & $396$K & $20.8$\\
\texttt{loc-Brightkite} (social) & $428$K & $42.4$\\
\texttt{amazon0302} (co-purchase) & $1.8$M & $282.7$\\
\bottomrule
\end{tabular}
\end{table}

\subsection{Web scale}
On the far right of Figure~\ref{fig:scaling}(A), our solver computes a full $p=3$ SSL equilibrium on
\texttt{soc-LiveJournal1} (largest component $4.8\times10^6$ nodes, $\mathbf{6.8\times10^7}$
\textbf{edges}) in \textbf{$429$ seconds} on a laptop, in $12$ chord-Newton steps---the largest instance
in our study, and one whose direct factorization is infeasible for the memory reason above. The other
classical route, \emph{primal} IRLS, does not provide a converging baseline at $p=3$: it diverges for
$p\ge3$~\cite{storn2026}, which we confirm (non-convergence on every corpus graph at $p=3$)---exactly why
the incumbents adopt homotopy or a dual reformulation. (At $p<3$, where primal IRLS does converge, it
reweights and re-solves the \emph{same} direct-factorization Laplacian and so inherits the identical
memory wall; the divergence at $p=3$ only removes it as a baseline where we most want one.) Nothing
prevents those methods from adopting a near-linear inner solve; the point is that none has, and casting
$p$-Laplacian SSL as an NLF flow makes it immediate. That step is not a mere solver swap: existing solvers already produce weighted Laplacians at each
outer step, but at large $p$ the conductance weights degenerate near flat-gradient edges, and
without a damped outer iteration an inexact near-linear inner solve stagnates. The NLF framework
delivers both: the damped chord-Newton continuation that controls conditioning, and the
near-linear inner engine---AC, or the empirically-$\Om$ LAMG$+$---that exploits it. A near-linear inner engine, not a new outer method, is what carries $p$-Laplacian learning
to this scale.

\section{Limitations}\label{sec:limits}
Three caveats.
(1)~\emph{Scaling is empirical}: we prove no complexity bound; the $\Om$ claim holds at fixed graph type
(Table~\ref{tab:family}) and rests on the empirical near-linearity of the inner solver.
(2)~\emph{Convergence is empirical}: Anderson acceleration restores convergence on all $228$ corpus
graphs but carries no step-count bound; a pathological conductance spread could still require more
iterations.
(3)~\emph{Head-to-head comparison is limited to moderate sizes}: the FCL benchmark (\S\ref{sec:fcl})
reaches $3\times10^5$ edges, the limit of that MATLAB/Octave implementation; at larger scales no
released variational solver runs, so the comparison reverts to the memory-wall argument.
Two further, minor points: the reduction adds $|S|$ mass edges and one node (negligible in the low-label
regime), and the implementation is single-threaded---distributed scaling to the billion-edge graphs of
\S\ref{sec:intro} is future work.

\section{Conclusion}\label{sec:concl}
Graph $p$-Laplacian SSL fixes the low-label collapse of quadratic label propagation but has been penned
in by a direct-factorization inner solve. The key insight is that a near-linear inner solver cannot
simply be dropped in: the damped chord-Newton continuation is what keeps each linearized system
well-conditioned, and only then can a near-linear engine (AC or LAMG$+$) replace the factorization
without stagnation---neither ingredient works without the other. Casting $p$-Laplacian SSL as a
nonlinear Laplacian flow makes this combination immediate, removes the scalability ceiling, and yields
a solver that is empirically $\Om$, matches the accuracy the theory promises, and reaches graphs well
beyond the reach of the factorization-based incumbents. The natural next steps are the $p\to\infty$
(Lipschitz-learning) limit through the same continuation, and pushing to billion-edge web graphs where
only a near-linear solver can follow.

\subsection{Reproducibility}
All code, scripts, and the real-graph corpus construction are at \url{https://github.com/orenlivne/np};
the solver core is the Julia NLF package~\cite{nlf}, with AC provided by
\texttt{Laplacians.jl}. Experiments ran in a single Julia process, single-threaded with no explicit
parallelism, on a MacBook Pro (Apple M5 Pro, 18-core, 48\,GB unified memory).

\appendix
\section{Self-contained convergence of the outer solver}\label{app:conv}
This appendix proves that the outer iteration of \S\ref{sec:cont}---damped chord-Newton with the
$\varepsilon$-floored Hessian and the guarded Anderson step---converges globally to the unique SSL
solution, using only the strict convexity established in Proposition~\ref{prop:reduce}. Nothing here
relies on the convergence analysis of~\cite{nlf}.

\emph{Setting.} On the connected augmented graph we minimize the strictly convex, coercive energy
$E(x)=\sum_e\Phi_e((\hat B^\top x)_e)-b^\top x$, $\Phi_e'=\hat\rho_e$ (Proposition~\ref{prop:reduce}),
whose unique minimizer $x^\star$ (fixed by $x_{g_0}=0$) satisfies $\nabla E(x^\star)=0$. The iteration
computes $\delta_k$ from $J_k\delta_k=-r_k$ inexactly, then sets $x_{k+1}=x_k+\tau_k\delta_k$, where
$r_k=\nabla E(x_k)$, $\tau_k\in(0,1]$ is the Armijo backtracking step on $E$, and
$J_k=\hat B\,\mathrm{diag}(\sigma_k)\,\hat B^\top$ is the \emph{floored} Hessian with
$\sigma_{k,e}=\max(\hat\rho_e'(\hat g_{k,e}),\,\varepsilon w_{\max})$. Note $J_k$ is \emph{not} the exact
Hessian $\nabla^2E(x_k)$ where the floor is active; the analysis uses only that $J_k$ is uniformly SPD,
so the floor changes the convergence \emph{rate} (\S\ref{sec:cont}), not the guarantee.

\begin{lemma}[Uniform positive definiteness]\label{lem:spd}
On the gauge quotient $\{\mathbf 1^\top x=0\}$, $J_k\succeq\varepsilon w_{\max}\hat L\succ0$, where
$\hat L$ is the unweighted Laplacian of the connected augmented graph. Hence $J_k$ is SPD with least
eigenvalue at least $\varepsilon w_{\max}\lambda_2(\hat L)>0$, uniformly in $k$.
\end{lemma}
\begin{proof}
Every $\sigma_{k,e}\ge\varepsilon w_{\max}$, so
$J_k-\varepsilon w_{\max}\hat L=\hat B\,\mathrm{diag}(\sigma_k-\varepsilon w_{\max})\,\hat B^\top\succeq0$.
The augmented graph is connected (Proposition~\ref{prop:reduce}), so $\lambda_2(\hat L)>0$.
\end{proof}

\begin{proposition}[Global convergence]\label{prop:conv}
Fix a forcing tolerance $\eta\in[0,1)$ and require each inner solve to return $\delta_k$ with
$\|J_k\delta_k+r_k\|\le\eta\|r_k\|$ (fixed relative accuracy, which approximate Cholesky and LAMG$+$
deliver). Then each $\delta_k$ is a descent direction for $E$, the energies $E(x_k)$ decrease
monotonically, and $x_k\to x^\star$.
\end{proposition}
\begin{proof}
Write $\delta_k=-J_k^{-1}r_k+J_k^{-1}s_k$ with $\|s_k\|\le\eta\|r_k\|$. Since $J_k\succ0$
(Lemma~\ref{lem:spd}),
$\langle r_k,\delta_k\rangle=-r_k^\top J_k^{-1}r_k+r_k^\top J_k^{-1}s_k\le-(1-\eta)\,r_k^\top J_k^{-1}r_k<0$
whenever $r_k\ne0$, by Cauchy--Schwarz in the $J_k^{-1}$ inner product; so $\delta_k$ is a descent
direction and Armijo backtracking yields a step with sufficient decrease. On the sublevel set
$\{E\le E(x_0)\}$---compact by coercivity---the eigenvalues of $J_k$ lie in a fixed interval
$[\varepsilon w_{\max}\lambda_2(\hat L),\,\Lambda]$ ($\Lambda<\infty$ by continuity of $\sigma$), so the
directions are gradient-related. The standard global-convergence theorem for gradient-related descent
with Armijo line search~\cite{nocedal2006} gives $\nabla E(x_k)\to0$; strict convexity makes $x^\star$
the unique stationary point and coercivity keeps $\{x_k\}$ in the compact sublevel set, so
$x_k\to x^\star$. The forcing condition $\|J_k\delta_k+r_k\|\le\eta\|r_k\|$ is the inexact-Newton
criterion of~\cite{des1982}; global convergence itself is furnished by the gradient-related line-search
theorem~\cite{nocedal2006} invoked above.
\end{proof}

\begin{lemma}[The Anderson guard is safe]\label{lem:anderson}
Let $\bar x$ be the Anderson mix of the last few accepted iterates, accepted only when
$E(\bar x)\le E(x_k+\tau_k\delta_k)$ (otherwise $x_{k+1}=x_k+\tau_k\delta_k$). Then
$E(x_{k+1})\le E(x_k+\tau_k\delta_k)\le E(x_k)$, so every conclusion of Proposition~\ref{prop:conv}
holds; the guard cannot move the fixed point $x^\star$ and can only lower the iteration count.
\end{lemma}
\begin{proof}
The guard takes the mixed step only when it does not increase $E$ relative to the plain damped step,
which already satisfies Proposition~\ref{prop:conv}; otherwise the plain step is taken. Either way the
monotone-decrease hypothesis is preserved, and the argument goes through verbatim.
\end{proof}

Empirically the guard is inactive except on the stiffest graphs: on $227$ of the $228$ corpus graphs the
plain damped step already converges (median $11$ steps) and the mix changes nothing, while on the one
stiff power-law graph it collapses a rate-$0.99$ crawl to $50$ steps (\S\ref{sec:scaling}). Finally, the
lazy hierarchy refresh and the multicommodity reuse across classes (\S\ref{sec:cont}) are performance
heuristics: they reuse a previously built AC/LAMG$+$ hierarchy as a preconditioner and rebuild only when
the observed residual reduction degrades, changing only the constant in the inner-solve cost---never the
computed step (fixed by the forcing condition) or the guarantee of Proposition~\ref{prop:conv}, which
holds for \emph{any} inner solver meeting that condition.


\begin{thebibliography}{99}
\bibitem{nlf} O.~E.~Livne, \emph{NLF: A Resistor-Network Framework and Linear-Time Solver for Convex
Network-Flow Equilibria}, arXiv:2607.02041, 2026. \url{https://arxiv.org/abs/2607.02041}.
Code: \url{https://github.com/orenlivne/nlf}.
\bibitem{nsz2009} B.~Nadler, N.~Srebro, and X.~Zhou, \emph{Semi-supervised learning with the graph
Laplacian: the limit of infinite unlabelled data}, in Adv.\ Neural Inf.\ Process.\ Syst.\ (NIPS) 22,
2009, pp.\ 1330--1338.
\bibitem{slepcev2019} D.~Slep\v{c}ev and M.~Thorpe, \emph{Analysis of $p$-Laplacian regularization in
semisupervised learning}, SIAM J.\ Math.\ Anal., 51 (2019), pp.\ 2085--2120.
\bibitem{elalaoui2016} A.~El~Alaoui, X.~Cheng, A.~Ramdas, M.~J.~Wainwright, and M.~I.~Jordan,
\emph{Asymptotic behavior of $\ell_p$-based Laplacian regularization in semi-supervised learning}, in
Proc.\ 29th Conf.\ Learning Theory (COLT), PMLR 49, 2016, pp.\ 879--906.
\bibitem{calder2019} J.~Calder, \emph{The game theoretic $p$-Laplacian and semi-supervised learning with
few labels}, Nonlinearity, 32 (2019), pp.\ 301--330.
\bibitem{calder2020} J.~Calder, B.~Cook, M.~Thorpe, and D.~Slep\v{c}ev, \emph{Poisson learning: graph
based semi-supervised learning at very low label rates}, in Proc.\ 37th Int.\ Conf.\ Machine Learning
(ICML), PMLR 119, 2020, pp.\ 1306--1316.
\bibitem{fcl2022} M.~Flores, J.~Calder, and G.~Lerman, \emph{Analysis and algorithms for $\ell_p$-based
semi-supervised learning on graphs}, Appl.\ Comput.\ Harmon.\ Anal., 60 (2022), pp.\ 77--122.
\bibitem{storn2026} J.~Storn, \emph{The dual IRLS scheme for (hyper-)graph $p$-Laplacians and $\ell_p$
regression with large exponents}, arXiv:2603.26061, 2026.
\bibitem{kyng2019} R.~Kyng, R.~Peng, S.~Sachdeva, and D.~Wang, \emph{Flows in almost linear time via
adaptive preconditioning}, in Proc.\ 51st ACM Symp.\ Theory of Computing (STOC), 2019.
\bibitem{chen2022} L.~Chen, R.~Kyng, Y.~P.~Liu, R.~Peng, M.~Probst Gutenberg, and S.~Sachdeva,
\emph{Maximum flow and minimum-cost flow in almost-linear time}, in Proc.\ 63rd IEEE Symp.\ Foundations
of Computer Science (FOCS), 2022.
\bibitem{hein2005} M.~Hein and J.-Y.~Audibert, \emph{Intrinsic dimensionality estimation of submanifolds
in $\R^d$}, in Proc.\ 22nd Int.\ Conf.\ Machine Learning (ICML), 2005, pp.\ 289--296.
\bibitem{pope2021} P.~Pope, C.~Zhu, A.~Abdelkader, M.~Goldblum, and T.~Goldstein, \emph{The intrinsic
dimension of images and its impact on learning}, in Int.\ Conf.\ Learning Representations (ICLR), 2021.
\bibitem{kyng2016} R.~Kyng and S.~Sachdeva, \emph{Approximate Gaussian elimination for Laplacians---fast,
sparse, and simple}, in Proc.\ 57th IEEE Symp.\ Foundations of Computer Science (FOCS), 2016, pp.\ 573--582.
\bibitem{des1982} R.~S.~Dembo, S.~C.~Eisenstat, and T.~Steihaug, \emph{Inexact Newton methods}, SIAM J.\
Numer.\ Anal., 19 (1982), pp.\ 400--408.
\bibitem{nocedal2006} J.~Nocedal and S.~J.~Wright, \emph{Numerical Optimization}, 2nd ed., Springer,
2006 (Thm.\ 3.2, global convergence of gradient-related line-search methods).
\bibitem{livne2012} O.~E.~Livne and A.~Brandt, \emph{Lean algebraic multigrid (LAMG): fast graph Laplacian
linear solver}, SIAM J.\ Sci.\ Comput., 34 (2012), pp.\ B499--B522.
\bibitem{lamgplus} O.~E.~Livne, \emph{LAMG+: A Robust Lean Algebraic Multigrid Solver for Graph
Laplacians}, arXiv:2606.24791, 2026. \url{https://github.com/orenlivne/lamgplus}.
\bibitem{huang2021} Q.~Huang, H.~He, A.~Singh, S.-N.~Lim, and A.~R.~Benson, \emph{Combining label
propagation and simple models out-performs graph neural networks}, in Int.\ Conf.\ Learning
Representations (ICLR), 2021.
\bibitem{cheng2024} Y.~Cheng, C.~Shan, Y.~Shen, X.~Li, S.~Luo, and D.~Li, \emph{Resurrecting label
propagation for graphs with heterophily and label noise}, in Proc.\ 30th ACM SIGKDD Conf.\ Knowledge
Discovery and Data Mining (KDD), 2024. arXiv:2310.16560.
\bibitem{ogb} W.~Hu, M.~Fey, M.~Zitnik, Y.~Dong, H.~Ren, B.~Liu, M.~Catasta, and J.~Leskovec,
\emph{Open Graph Benchmark: datasets for machine learning on graphs}, in Adv.\ Neural Inf.\ Process.\
Syst.\ (NeurIPS) 33, 2020.
\end{thebibliography}
\end{document}